%% file: main.tex
\documentclass[12pt]{colt2019} % Include author names

\usepackage{amsfonts}
\usepackage{bbm}
\usepackage{enumitem}
\usepackage{eufrak}
\usepackage{fullpage}
\usepackage{graphicx}
\usepackage{mathtools}
\usepackage{soul}
\usepackage{parskip}
\usepackage{st}

\title[Exponentiated Gradient Meets Gradient Descent]{Exponentiated Gradient \st{vs.} Meets Gradient Descent}
\usepackage{times}

\renewcommand{\cite}{\citep}

\coltauthor{\!\!
\Name{Udaya Ghai} \Email{ughai@cs.princeton.edu}\\
\Name{Elad Hazan} \Email{ehazan@cs.princeton.edu }\\
\Name{Yoram Singer} \Email{y.s@cs.princeton.edu}\\
\addr Google AI Princeton$^{\heartsuit}$ \& Princeton University}

\begin{document}

    \maketitle

    \input{abstract.tex}

    \medskip


    \input{introduction.tex}
    \input{problem_setting.tex}

    \input{divergence.tex}
    \input{matrix_divergence.tex}
    \input{experiments.tex}
    \input{discussion.tex}

    \acks{We thank Orestis Plevrakis for enlightening discussion.}

    %\newpage
    \bibliography{main}

    %Appendix sections
    \appendix
    \input{omd_background.tex}
    \input{egpm.tex}

\end{document}

%% file: abstract.tex
\begin{abstract}
The (stochastic) gradient descent and the multiplicative update method are
probably the most popular algorithms in machine learning. We introduce and
study a new regularization which provides a unification of the additive and
multiplicative updates. This regularization is derived from an hyperbolic
analogue of the entropy function, which we call hypentropy. It is motivated
by a natural extension of the multiplicative update to negative numbers.
The hypentropy has a natural spectral counterpart which we use to derive a
family of matrix-based updates that bridge gradient methods and the
multiplicative method for matrices. While the latter is only applicable to
positive semi-definite matrices, the spectral hypentropy method can
naturally be used with general rectangular matrices. We analyze the new
family of updates by deriving tight regret bounds. We study empirically the
applicability of the new update for settings such as multiclass learning, in
which the parameters constitute a general rectangular matrix.
\end{abstract}

%\eh{Elad: add discussion of improvement of online matrix prediction, and
%improved bounds for the matrix prediction in general over matrix MW} 
%\eh{check potential tighter regret bounds for online variance minimization,
%and online leading eigenvector problem} 

%% file: introduction.tex
\section{Introduction} \label{intro:sec}

Algorithms for online learning can morally be divided into two camps. On one
side is the additive gradient update. Additive gradient-based stochastic
methods are the most commonly used approach for learning the parameters of
shallow and deep models alike. On the other side stands the multiplicative
update method. It is somewhat less glamorous, nonetheless a fundamental
primitive in game theory and machine learning, and was rediscovered repeatedly
in a variety of algorithmic settings~\cite{MWU}.
Both additive and multiplicative updates can be seen as special cases of a more
general technique of learning with {\it regularization}. General frameworks for
regularization were developed in online learning, dubbed
Follow-The-Regularized-Leader and in convex optimization as the {\it Mirrored
Decent} algorithms, see more below.

Notable attempts were made to unify different regularization techniques, in
particular between the multiplicative and additive update
methods~\cite{kivinen1997exponentiated}. For example,
AdaGrad~\cite{duchi2011adaptive} stemmed from a theoretical study of learning
the best regularization in hindsight. As the name implies, the $p$-norm
update~\cite{grove2001general,Gentile2003} uses the squared $p$-norm of the parameters as a
regularization.
By varying the order of the norm between regret bounds that
are characteristic of additive and multiplicative updates.

We study a new, arguably more natural, family of regularization which
``interpolates'' between additive and multiplicative forms. We analyze its
performance both experimentally, and theoretically to obtain tight regret bounds
in the online learning paradigm. The motivation for this interpolation stems
from the extension of the multiplicative update to negative weights. Instead of
using the so called EG$\pm$ trick", a term coined by~\citet{MWpc},
which simulates arbitrary weights through duplication to positive and negative
components, we use a direct approach. To do so we introduce the hyperbolic
regularization with a single temperature-like hyperparameter. Varying the
hyperparameter yields regret bounds that translate between those akin to
additive, and multiplicative, update rules.

As a natural next step, we investigate the spectral analogue of the
hypentropy function. We show that the spectral hypentropy is strongly-convex
with respect to the Euclidean or trace norms, again as a function of the
single interpolation parameter. The spectral hypentropy yields updates that
can be viewed as interpolation between gradient descent rule and
matrix multiplicative update~\cite{tsuda2005matrix, arora2007combinatorial}.

The standard matrix multiplicative update rule applies only to positive
semi-definite matrices. Standard extensions to square and more general
matrices increase the dimensionality~\cite{hazan2012near}. Moreover, the the
regret bounds scale as $O(\sqrt{T \log (m+n)} $ for $m \times n$ matrices.
In contrast, the spectral hypentropy regularization is defined for
arbitrary, rectangular, matrices. Moreover, the hypentropy-based update in
better regret bounds of $O( \sqrt{T \log \min \{m,n\}}) $, matching the best
known bounds in~\citet{kakade2012regularization}.

\paragraph{Related work}
For background on the multiplicative updates method and its use in machine
learning and algorithmic design, see~\citet{MWU}. The matrix version of
multiplicative updates method was proposed in~\citet{tsuda2005matrix} and later
in~\citet{arora2007combinatorial}. The study of the interplay between additive
and multiplicative updates was initiated in the influential paper of
\citet{kivinen1997exponentiated}. Generalizations of multiplicative updates to
negative weights were studied in the context of the Winnow algorithm and mistake
bounds in \citet{warmuth2007winnowing, grove2001general}. The latter paper also
introduced the $p$-norm algorithm which was further developed
in~\citet{Gentile2003}. The generalization of the p-norm regularization to
matrices was studied in \citet{kakade2012regularization}.

\paragraph{Organization of paper}
$\HU$ and $\SHU$ are mirror descent algorithms using the hypentropy and
spectral hypentropy regularization functions defined in Sec.~\ref{div:sec} and
Sec.~\ref{matrix_div:sec} respectively. These sections explore the geometric
properties of these new regularization functions and provide regret analysis.
Experimental results which underscore the applicability of $\HU$ and $\SHU$ are described in
Sec.~\ref{experiments:sec}. A thorough description of mirror descent is given
for completeness in App.~\ref{appendex1:sec}. The view of $\EGpm$ as an
adaptive variant of $\HU$ is explored in App.~\ref{egpm_appendix:sec}.

%% file: problem_setting.tex
\section{Problem Setting} \label{setting:sec}

\paragraph{Notation.}
Vectors are denoted by bold-face letters, e.g. $\bw$. The zero vector and
the all ones vector are denoted by $\bzero$ and $\bone$ respectively. We
denote a ball of radius $1$ with respect to the $p$-norm in $\RR^d$ as
$B_{p} = \{\bx \in \RR^d: \|\bx\|_p \leq 1\}$. For simplicity of the
presentation in the sequel we assume that weights are confined to the unit
ball. Our results generalize straightforwardly to arbitrary radii.

Matrices are denoted by capitalized bold-face letters, e.g. $\bX$.
We denote the space of real matrices of size $m \times n$ as $\RR^{m\times n}$
and symmetric matrices of size $d \times d$ as $\SS^d$. For a matrix
$\bX\in \RR^{m \times n}$, we denote the vector of singular values
$\sigma(\bX) = (\sigma_1, \sigma_2, \dots ,\sigma_l)$ where
$\sigma_1 \geq \sigma_2 \geq \cdots \geq \sigma_l \geq 0$ and
$l = \min\{m,n\}$. Analogously, for $\bX\in \SS^d$, we denote the vector of
its eigenvalues as $\lambda(\bX) = (\lambda_1, \lambda_2, \dots ,\lambda_d)$.
We use $\|\bX\|_p \eqdef \|\sigma(\bX)\|_p$ to represent the Schatten $p$-norm
of a matrix, namely, the $p$-norm of the vector of eigenvalues. We refer to the
Schatten norm for $p=1$ as the trace-norm. Note that the notation of the
spectral norm of $\bX$ is $\|\bX\|_\infty$.
We denote the ball of radius $\tau$
with respect to the trace-norm as
$\btr{\tau} = \{\bX \in \RR^{m \times n} : \|\bX\|_1 \leq \tau\}$.
We also define the intersection of a ball and the positive orthant as
$B^{+}_{p} = B_{p} \cap \RR^d_{+}$. We use $(x)_+$ to denote $\max(x,0)$.
We denote by $\|\bx\|_{*}$ the dual norm of $\bx$,
$\|\bx\|_{*} \eqdef \sup\{\bz^\top\bx \, | \, \|\bz\|\leq 1\}$.

\paragraph{Online Convex Optimization.} In online convex
optimization~\cite{cesa2006prediction, hazan2016introduction,
shalev2012online}, a learner iteratively chooses a vector from a convex set
$\calK \subset \RR^d$. We denote the total number of rounds as $T$. In each
round, the learner commits to a choice $\bw_t \in \calK$. After committing to
this choice, a convex loss function $\ell_t: \calK \rightarrow \RR$ is
revealed and the learner incurs a loss $\ell_t(\bw_t)$. The most common
performance objective of an online learning algorithm $\calA$ is regret.
Regret is defined to be the total loss incurred by the algorithm with respect
to the loss of the best fixed single prediction found in hindsight. Formally,
the regret of a learning algorithm $\calA$ is defined as,
\begin{align*}
    \regret_T(\calA) &\eqdef
    \sup_{\ell_1 \dots \ell_t} \bigg \{\sum_{t=1}^T\ell_t(\bw_t) -
    \min_{\bw^{*}\in \calK}\sum_{t=1}^T\ell_t(\bw^{*}) \bigg\} ~.
\end{align*}

%% file: divergence.tex
\section{\Hypent Divergence}
\label{div:sec}

We begin by defining the $\beta$-hyperbolic entropy, denoted $\phi_{\beta}$.
\begin{definition}[Hyperbolic-Entropy]
    For all $\beta > 0$, let $\phi_{\beta} : \RR^{d} \rightarrow \RR$ be defined as,
    \begin{align*}
        \phi_{\beta}(\bx) =
        \sum_{i=1}^d \Big(x_i \arcsinh\Big(\frac{x_i}{\beta}\Big) - \sqrt{x^2_i +\beta^2}\Big)~.
    \end{align*}
\end{definition}
Alternatively, we can view $\phi_\beta(\bx)$ as the sum of scalar functions,
$\phi_{\beta}(\bx) = \sum_{i=1}^d \phi_\beta(x_i)$, each of which satisfies,\vspace{-8pt}
\begin{align}
    \label{second_deriv:eq}
    \phi''_\beta(x) = \frac{1}{\sqrt{x^2 + \beta^2}} ~.
\end{align}

For brevity and clarity, we use the shorthand \hypent for $\phi_{\beta}$. Given the
\hypent function, we derive its associated Bregman divergence,
the relative \hypent as,
\begin{align*}
    D^{\beta}_{\phi}\infdivx{\bx}{\by} =
    {}& \phi_{\beta}(\bx) - \phi_{\beta}(\by) - \ip{\nabla \phi_{\beta}(\by)}{\bx - \by}\\
    ={}& \sum_{i=1}^d \bigg[x_i(\arcsinh\Big(\frac{x_i}{\beta}\Big) - \arcsinh\Big(\frac{y_i}{\beta}\Big)) - \sqrt{x^2_i +\beta^2} + \sqrt{y^2_i +
    \beta^2}\bigg]~.
\end{align*}
As we vary $\beta$, the relative \hypent interpolates between the squared
Euclidean distance and the relative entropy. The potentials for these divergences
are sums of element-wise scalar functions, for simplicity we view them as
scalar functions. The interpolation properties of hypentropy
can be seen in Figure~\ref{div:fig}. As $\beta$ approaches $0$,
we see that $\nabla^2\phi_{\beta}$ approaches $\frac{1}{|x|}$. When working
only over the positive orthant, as is the case with entropic regularization,
the hypentropy second derivative converges to the second
derivative of the negative entropy. On the other hand, as
$\beta$ grows much larger than $x$, we see $\sqrt{\beta^2 + x^2} \approx \beta$.
Therefore, for larger $\beta$, $\nabla^2\phi_{\beta}$ is essentially a constant.
In this regime hypentropy behaves like a scaled squared euclidean distance.

\begin{figure}
    \scalebox{0.6}{\begin{tabular}[c]{||c| c| c| c||}
                       \hline
                       & Square & Entropy & Hypentropy \\ [0.5ex]
                       \hline
                       \hline
                       $\phi(x)$ & $\frac12 {x^2}$ & $x\log(x) -x$ & \makecell{ $x\arcsinh(\frac{x}{\beta})$\\ $- \sqrt{x^2 +\beta^2}$}\\
                       \hline
                       $\nabla\phi(x)$ & $x$ & $\log (x)$ & $\arcsinh(\frac{x}{\beta})$ \\
                       \hline
                       $\nabla^2\phi(x)$ & $1$ & $\frac{1}{x}$ & $\frac{1}{\sqrt{x^2 + \beta^2}}$ \\[1ex]
                       \hline
    \end{tabular}}
    $\vcenter{\hbox{\includegraphics[width=0.33\textwidth]{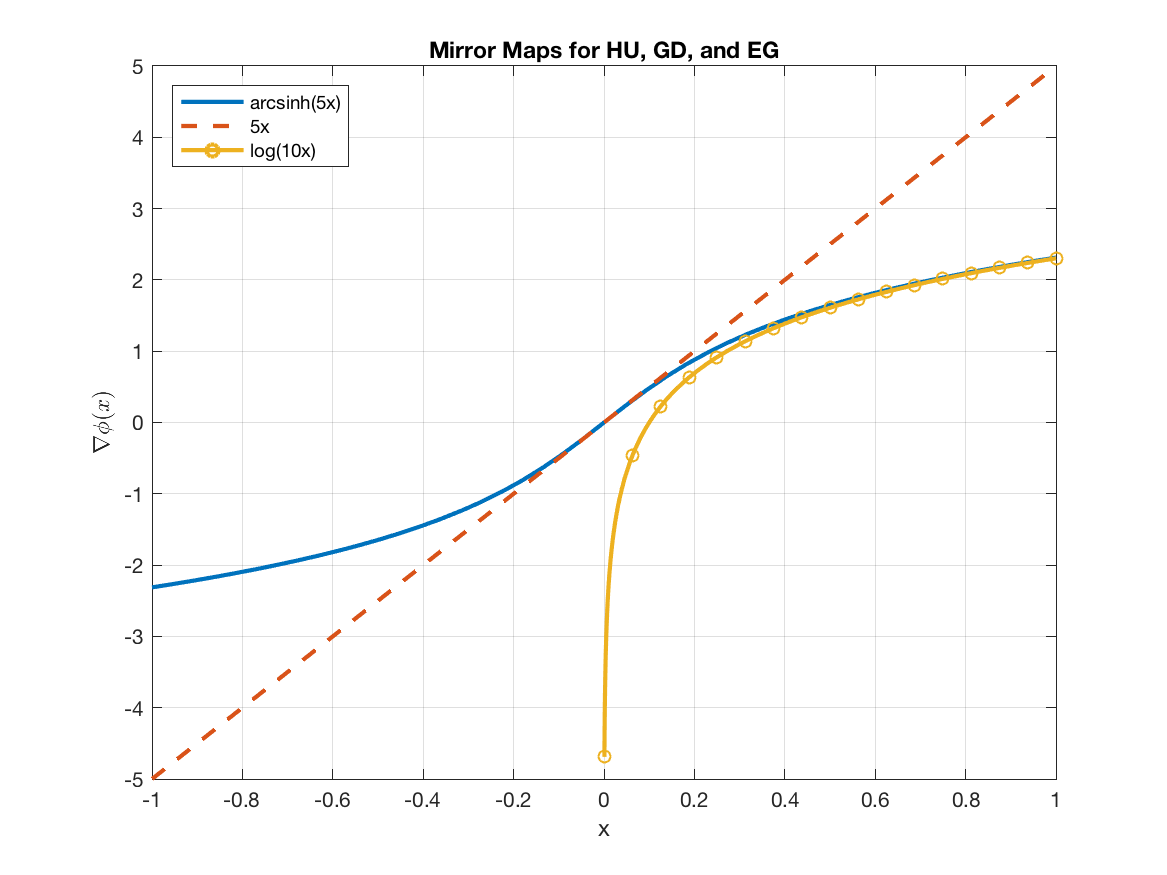}}}$
    $\vcenter{\hbox{\includegraphics[width=0.33\textwidth]{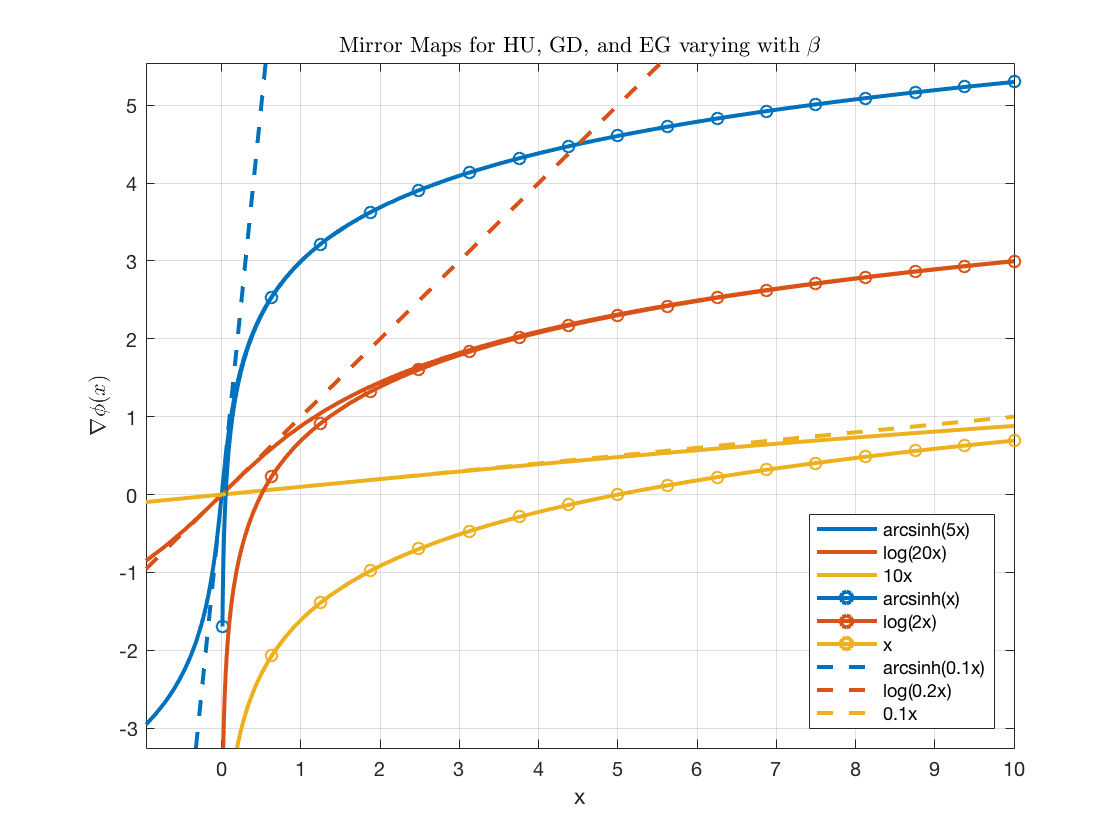}}}$
    \caption{Left:Classical and new divergences. Center: scalar versions of mirror maps used in $\HU(\beta=\frac{1}{5})$,
    GD, and EG Algorithms depicted in blue, red, and yellow respectively.
    EG is only defined for positive values of $x$. Right: $\beta$ is varied between
    $0.1, 1,$ and $10$ and is similarly depicted along with its linear and
    logarithmic limits. }
    \label{div:fig}
\end{figure}

From a mirror descent perspective of mirror descent (see Sec.~\ref{appendex1:sec}), it makes sense to look at the mirror map, the gradient
of the $\phi$ which defines the dual space where additive gradient updates
take place. Weights are mapped into the dual space via the mirror map
$\nabla \phi : \RR^d \rightarrow \RR^d$ and mapped back into the
primal space via $\nabla \phi^{*}$. Gradient-Descent (GD) can be framed as
mirror descent using the squared euclidean norm potential while
Exponentiated-Gradient (EG) amounts to mirror descent using the entropy
potential. The \Hypent Update (HU) uses the mirror
map $\nabla \phi_{\beta}$. As can be seen from Figure \ref{div:fig},
for sufficiently large weights, the \hypent mirror map behaves like
$\log(x)$, namely, the EG mirror map. In contrast, $\arcsinh(x/\beta)$ is
linear for small weights, and thus behave like GD. Large values of $\beta$
correspond to a slower transition from the linear regime to the logarithmic
regime of the mirror map.

The regret analysis of GD and EG depend on geometric properties of the
divergence related to the $2$-norm and $1$-norm respectively. Given this
connection, it is useful to analyze the properties of \hypent with respect
to both the $1$-norm and $2$-norm. Recall that a function $f$ is $\alpha$-strongly convex
with respect to a norm $\|\cdot\|$ on $\calK$ if,
$$\forall \bx, \by \in \calK, f(\bx) - f(\by) - \nabla f(\by)(\bx-\by)
\geq \frac{\alpha}{2} \| \bx-\by\|^2.$$
For convenience, we use the following second order characterization of
strong-convexity from Thm.~ 3~in~\cite{STRCVXNOTES}.
\begin{lemma}
    \label{second_strcvx:lemma}
    Let $\calK$ be a convex subset of some finite vector space $\calX$. A twice differentiable
    function $f : \calK \rightarrow \RR$ is $\alpha$-strongly convex with respect to a
    norm $\|\cdot\|$ iff
    \begin{align*}
        \inf_{\bx \in \calK,\by \in \calX: \|\by\| =1} \by^{\top}\,\nabla^2\phi(\bx)\,\by \geq \alpha ~.
    \end{align*}
\end{lemma}
We next prove elementary properties of $\phi_\beta$.
\begin{lemma}
    \label{strcvx_2:lemma} The function $\phi_\beta$ is
    $(1+ \beta)^{-1}$-strongly-convex over $B_2$ w.r.t the $2$-norm.
\end{lemma}
\begin{proof}
    To prove the first part, note that from~\eqref{second_deriv:eq} we get that
    the Hessian is the diagonal matrix,
    \begin{align*}
        \nabla^2 \phi_{\beta}(\bx) = \diag \Big[\tfrac{1}{\sqrt{x_1^2 + \beta^2}}, \ldots  \tfrac{1}{\sqrt{x_d^2 + \beta^2}}\Big]
    \end{align*}
    Strong convexity follows from the diagonal structure of the Hessian, whose
    smallest eigenvalue is
    \begin{align*}
        \frac{1}{\sqrt{x^2 + \beta^2}} \geq \frac{1}{\sqrt{1+ \beta^2}} \geq \frac{1}{1+\beta} ~.
    \end{align*}
\end{proof}\vspace{-16pt}
\begin{lemma}
    \label{strcvx_1:lemma} The function $\phi_\beta$ is
    $(1 + \beta d)^{-1}$-strongly-convex over $B_1$ w.r.t. the $1$-norm.
\end{lemma}
\begin{proof}
    We work with the characterization provided in Lemma~\ref{second_strcvx:lemma},
    \begingroup
    \allowdisplaybreaks
    \begin{align*}
        \inf_{\bx\in B_1, \|\by\|_1 =1} & \by^{\top}\nabla^2\phi(\bx)\by ~ =
        \inf_{\bx\in B_1, \|\by\|_1 =1}\sum_{i=1}^d \frac{y_i^2}{\sqrt{ \beta^2 + x_i^2}} & \big[\mbox{Equation}~\eqref{second_deriv:eq}\big]\\
        & =\inf_{\bx\in B_1, \|\by\|_1 =1}
        \frac{1}{\sum_{i=1}^d \sqrt{ \beta^2 + x_i^2}} \,
        \bigg(\sum_{i=1}^d \frac{y_i^2}{\sqrt{ \beta^2 + x_i^2}}\bigg) \,
        \bigg(\sum_{i=1}^d \sqrt{ \beta^2 + x_i^2}\bigg) \\
        & \geq\inf_{\bx\in B_1, \|\by\|_1 =1}
        \frac{1}{\sum_{i=1}^d \sqrt{ \beta^2 + x_i^2}}
        \bigg(\sum_{i=1}^d \sqrt{y_i^2}\bigg)^2 & \big[\mbox{Cauchy- Schwarz}\big]\\
        & =\inf_{\bx\in B_1, \|\by\|_1 =1}\frac{1}{\sum_{i=1}^d \sqrt{ \beta^2 + x_i^2}}\|\by\|^2_1
        \geq\frac{1}{\sum_{i=1}^d (\beta + |x_i|)}\geq \frac{1}{1 + \beta d} ~.
    \end{align*}
    \endgroup
\end{proof}
We next introduce a generalized notion of diameter and use it to prove properties
of $\phi_\beta$.
\begin{definition}
    \label{diam:def}
    The diameter of a convex set $\calK$ with respect to $\phi$ is,
    $
    \diam_{\phi}(\calK) \eqdef \sup_{\bx \in \calK}
    D_{\phi}\infdivx{\bx}{\bzero} ~.
    $
\end{definition}

Whenever implied by the context we omit the potential $\phi$ from the diameter.
Before we consider two specific diameters below, we bound the diameter in general as follows,
\begin{align*}
    D^{\beta}_{\phi}\infdivx{\bx}{\bzero} &= \phi_{\beta}(\bx) -
    \phi_{\beta}(\bzero) \\
    &= \sum_{i=1}^d \Big(x_i \arcsinh({x_i}/{\beta}) - \sqrt{x^2_i +\beta^2}\Big) + \beta d\\
    &\leq \sum_{i=1}^d x_i \arcsinh({x_i}/{\beta})\\
    &= \sum_{i=1}^d |x_i| \log\left(\frac{1}{\beta}\left(\sqrt{x^2_i + \beta^2} + |x_i|\right)\right)
\end{align*}
Thus, without loss of generality, we can assume that $\bx$ lies in the
positive orthant. We next bound the diameter of $B_2$ as follows,
\begin{align}
    \diam(B_2)
    &\leq \sum_{i=1}^d x_i \log{\left(\frac{1}{\beta}\left(\sqrt{x^2_i +
    \beta^2} + x_i\right)\right)} \nonumber \\
    % &\leq \sum_{i=1}^d x_i \log{\Big(\frac{\beta + 2x_i}{\beta}\Big)} \nonumber
    &\leq \sum_{i=1}^d x_i \log{\Big( 1 + \frac{2x_i}{\beta}\Big)} \nonumber\\
    &\leq \sum_{i=1}^d \frac{2x^2_i}{\beta} = \frac{2\|\bx\|^2_2}{\beta} \leq \frac{2}{\beta} \label{l2_diameter:eq}~.
\end{align}
For $\beta\leq 1$ and $\bx\in B_1$ it holds that,
${\sqrt{x^2_i + \beta^2} + x_i} \leq \sqrt{2} + 1$.
Hence, for $\beta\leq 1$, we have
\begin{align}
    \label{l1_diameter:eq}
    \diam(B_1)
    &\leq \sum_{i=1}^d x_i \log{\Big(\frac{1+\sqrt{2}}{\beta}\Big)}
    =    \|\bx\|_1 \log{\Big(\frac{1+ \sqrt{2}}{\beta}\Big)}
    \leq \log{\Big(\frac{3}{\beta}\Big)}~.
\end{align}

\subsection{HU algorithm}

We next describe an OCO algorithm over a convex domain $\calK \subseteq \RR^d$.
%%%%%%%%%%%%%%%%%%%%%%%%%%%%%%%%%%%%%%%%%%%%%%%%%%%%%%%%%%%%%%%%%%%%%%%%%%%%%%%
\begin{algorithm2e}[ht]
    \label{hu:algorithm}
    \SetAlgoLined
    \SetKw{KwBy}{by}
    \KwIn{$\eta >0, \beta > 0$, convex domain $\calK \subseteq \RR^{d}$}
    Initialize weight vector $\bw^{1} = \bzero$\;
    \For{$i=1$ \KwTo $T$}{
    (a) Predict $\bw^{t}$ ~ ~
    (b) Incur loss $\ell_t(\bw_t)$ ~ ~
    (c) Calculate $\bg^t =  \nabla \ell_t(\bw^t)$ \;
    \smallskip
    Update:
    % $ \begin{align}\label{SHU_update:eq}
    $  \bw^{t+\frac{1}{2}} =
    \beta \sinh{\Big(\arcsinh{\Big(\frac{\bw^{t}}{\beta}\Big)} - \eta
    \bg^{t}}\Big) $\;
    % \end{align}
    \smallskip
    Project onto $\calK$:
    % \begin{align*}
    $\bw^{t+1} = \displaystyle
    \argmin_{\bv \in \calK}
    D^{\beta}_{\phi}\infdivx{\bv}{\bw^{t+\frac{1}{2}}} \label{HU_update:eq}$
    % \end{align*}
    }
    \caption{\Hypent Update (HU)}
\end{algorithm2e}
%%%%%%%%%%%%%%%%%%%%%%%%%%%%%%%%%%%%%%%%%%%%%%%%%%%%%%%%%%%%%%%%%%%%%%%%%%%%%%%
%The Hypentropy Update (HU) algorithm starts at $\bw = \bzero$, with a fixed
%learning rate $\eta$ and hyperparameter $\beta$ and performs the update,
%\begin{align}
%    \label{HU_update:eq}
%    w^{t+\frac{1}{2}}_i &=
%    \beta \sinh{\Big(\arcsinh{\Big(\frac{w^{t}_i}{\beta}\Big)} - \eta g^{t}_i}\Big)\\
%    \bw^{t+1} &= \argmin_{\bv \in \calK} D^{\beta}_{\phi}\infdivx{\bv}{\bw^{t+\frac{1}{2}}} ~,
%    \nonumber
%\end{align}
%where $\bg^t =  \nabla \ell_t(\bw^t)$.

$\HU$ is an instance of OMD with divergence $D^{\beta}_{\phi}\infdivx{\cdot\!}{\!\cdot}$.
We provide a simple regret analysis that follows directly from the geometric properties derived above.
The following theorem allows us to bound the regret of an OMD algorithm in terms of the diameter and strong convexity.

\begin{theorem}
    \label{omd:theorem}
    Assume that $R: \calK \rightarrow \RR$ is $\mu$-strongly convex in respect to a
    norm $\|\cdot\|$ whose dual norm is $\|\cdot\|_{*}$. Assume that the diameter
    of $\calK$ is bounded, $\diam_R(\calK) \leq D$. Last, assume that
    $\forall{}t,\\|\bg^{t}\|_{*}\leq{}G$, then the regret boound of $\HU$
    and learning rate $\eta = \sqrt{{2\mu D}/({TG^2})}$ satisfies
    \begin{align*}
        \regret_T \leq  2\sqrt{2 \mu^{-1} DTG^2}~.
    \end{align*}
\end{theorem}
This follows from the more general Theorem~\ref{omd:theorem_general}.
We next provide regret bounds for $\HU$ over $B_1$ and $B_2$.

\medskip
\begin{theorem}[Additive Regret]
    \label{l2regret:theorem}
    Let $\bw \in B_2$ and assume that for all $t$, $\|\bg_t\|_2 \leq G_{2}$.
    Setting $\beta \geq 1$,
    $$\eta = \frac{1}{G_{2}}\sqrt{\frac{1}{\beta(\beta+1)T}}
    ~,~ \mbox{ yields } ~ ~ \regret_T(\HU) \leq 4G_2\sqrt{T} ~. $$
\end{theorem}
\begin{proof}
    Applying Lemma~\ref{strcvx_2:lemma} and the diameter bounds from \eqref{l2_diameter:eq}
    to Theorem~\ref{omd:theorem} yields,
    \begin{align*}
        \regret_T \leq 2\sqrt{2(1 + \frac{1}{\beta})TG^2_2} \leq 4G_2\sqrt{T} ~.
    \end{align*}
    The final inequality follows from the condition that $\beta \geq 1$.
\end{proof}

\begin{theorem}[Multiplicative Regret]
    \label{l1regret:theorem}
    Let $\bw \in B_1$ and assume that for all $t$, $\|\bg_t\|_{\infty} \leq G_{\infty}$.
    Setting for $\beta \leq 1$
    $$\eta = \frac{1}{G_{\infty}}
    \sqrt{\frac{\log{\big(\tfrac{3}{\beta}\big)}}{2T(1+\beta d)}}
    ~,~ \mbox{ yields } ~ ~
    \regret_T(\HU) \leq 3G_{\infty} \sqrt{T (1+ \beta d)\log{(\tfrac{3}{\beta})}} ~. $$
\end{theorem}

\begin{proof}
    Applying Lemma~\ref{strcvx_1:lemma} and the diameter bounds from \eqref{l1_diameter:eq} to Theorem~\ref{omd:theorem} yields,
    \begin{align*}
        \regret_T \leq 2G_{\infty} \sqrt{2T (1+ \beta d)\log{\big(\tfrac{3}{\beta}\big)}}\leq 3G_{\infty} \sqrt{T (1+ \beta d)\log{(\tfrac{3}{\beta})}}~.
    \end{align*}
\end{proof}\vspace{-32pt}

%% file: matrix_divergence.tex
\section{Spectral Hyperbolic Divergence} \label{matrix_div:sec}
In this section, the focus is on using \hypent as a spectral regularization
function. We show that the matrix version of $\HU$ is strongly convex with
respect to the trace norm. Our proof technique of strong convexity is a
roundabout for the matrix potential. The proof works by showing that the
conjugate potential function is smooth with respect to the spectral norm
(the dual of the trace norm). The duality of smoothness and strong convexity
is then used to show strong convexity.

\subsection{Matrix Functions}
We are concerned with potential functions that act on the singular values of a
matrix. For an even scalar function, $f:\RR \rightarrow \RR^{+}$, consider
the trace function,
\begin{align}
    \label{singular_func:eq}
    F(\bX) = (f \circ \sigma)(\bX) =
    \sum_{i=1}^d f(\sigma_i) =
    \Tr\left(f\left(\sqrt{\bX^{\top}\bX}\right)\right)~,
\end{align}
where we overload the notation for $f$ and denote $f(\bv) = \sum_{i=1}^d f(v_i)$.
For $\bX\in \SS^d$ we use
\begin{align}
    \label{singular_func_b:eq}
    F(\bX) = (f \circ \lambda)(\bX)= \Tr(f(\bX)) ~.
\end{align}
Here $f(\bX)$ represents the standard lifting of a scalar function to a square
matrix, where $f$ acts on the vector of eigenvalues, namely,
\begin{align*}
    \bX=\bU\diag\left[\lambda(\bX)\right]\bU^{\top} \quad\Rightarrow\quad
    f(\bX) = \bU \diag\left[f(\lambda(\bX))\right]\bU^{\top} ~.
\end{align*}

We also use the gradient of a trace-function in our analysis. The following
result from Thm.~14~in~\cite{strcvxdual} shows how to compute a gradient
using a singular value decomposition.
\begin{theorem}
    Let $\bX\in\RR^{m \times n}$ and $F: \RR^{m \times n} \rightarrow \RR^{+}$ be
    defined as above, then
    \begin{align*}
        \nabla F(\bX) = f'(\bX) ~ .
    \end{align*}
\end{theorem}

We also make use of the \emph{Fenchel} conjugate functions. Consider a convex
function $f:\calX \rightarrow \RR$ defined on a {\em finite} vector space $\calX$
endowed with an inner product $\ip{\cdot}{\cdot}$. The conjugate of $f$ is
defined as follows.

\begin{definition}
    The conjugate $f^{*}: \calX \rightarrow \RR$ of a convex function
    $f:\calX \rightarrow \RR$ is
    \begin{align*}
        f^{*}(\bz) = \sup_{\bx \in \calX} \ip{\bx}{\bz} - f(\bx) ~.
    \end{align*}
\end{definition}
In this section, we use the space of matrices (either $\RR^{m\times n}$ or
$\SS^d)$ with the inner product $\ip{\bX}{\bY} = \Tr(\bX^{\top}\bY)$. Thus,
the dual space of $\calX$ is $\calX$ itself and $f^*$ is defined over $\calX$.

\smallskip

We need to relate the conjugate of a trace function to that of a scalar
function. This is achieved by the following result, restated from
Thm.~12~\cite{strcvxdual}. The theorem implies that the conjugate of a
singular-values function is the singular-values function lifted from the
conjugate of the scalar function.
\begin{theorem}
    \label{matrix_conj:theorem}
    Let $F = (f \circ \sigma)$ be defined as in \eqref{singular_func:eq},
    then $F^{*} = (f^{*} \circ \sigma)$.
\end{theorem}

\subsection{Duality of Strong Convexity and Smoothness}
Recall that a function $f$ is $L$-smooth with respect to a
norm $\|\cdot\|$ on $\calK$ if,
\begin{align*}
    \forall \bx, \by \in \calK, ~~
    f(\bx) - f(\by) - \nabla f(\by)(\bx-\by) \leq \frac{L}{2} \| \bx-\by\|^2 ~.
\end{align*}

For convenience, we use the following second order characterization of
smoothness which is an analogue of Lemma~\ref{second_strcvx:lemma}.
\begin{lemma}
    \label{second_strsmooth:lemma}
    A twice differentiable function $f : \calX \rightarrow \RR$ is locally
    $L$-smooth with respect to $\|\cdot\|$ at $\bx$ iff
    \begin{align*}
        \sup_{\by \in \calX :\|\by\| =1} \by^{\top}\nabla^2\phi(\bx)\,\by \leq L~.
    \end{align*}
\end{lemma}
Strong convexity and smoothness are dual notions in the sense that $f$ is
$\alpha$-strongly convex with respect to a norm $\|\cdot\|$ iff its Fenchel
conjugate $f^{*}$ is ${\alpha}^{-1}$-smooth with respect to the dual norm
$\|\cdot\|_{*}$.

For the matrix variant of \hypent we find it easier to show smoothness of the
conjugate rather than strong convexity directly. Unfortunately, as we see in the
sequel, the conjugate function is not smooth everywhere. Therefore, we would
need a local variant of the duality of strong convexity and smoothness. In the
context of mirror descent with mirror map $\nabla \phi$, we show that $\phi$
is strongly convex over $\calK$ if $\phi^{*}$ is locally smooth at all points
within the image of the mirror map. Formally, we have the following lemma.
In the following we use the standard notation for image of vector functions,
$\nabla  f(S) = \{ y | \nabla f (x) =y , x \in S \} $,

\begin{lemma}[Local duality of smoothness and strong convexity]
    \label{local_dual_smooth:lemma}
    Let $\calK \subseteq \RR^d$ be an open convex set and $\|\cdot\|$ be a norm with dual norm $\|\cdot\|_{*}$. Let $\phi: \RR^d \rightarrow \RR$ be twice differentiable, closed and convex function. Suppose the Fenchel conjugate $\phi^{*}: \RR^d \rightarrow \RR$ is locally $L$ smooth with respect to $\|\cdot\|_{*}$ at all points in $\calC = \nabla \phi(\calK)$. Then, $\phi$ is
    $\frac{1}{L}$ strongly convex with respect to $\|\cdot\|$ over $\calK$.
\end{lemma}

\begin{proof}
    It suffices to show that for any $\bx \in \calK$, $\phi$ is locally
    $\frac{1}{L}$-strongly convex with respect to $\|\cdot\|$ at $\bx$ if
    $\phi^{*}$ is locally $L$-smooth at $\bx^{*} = \nabla \phi(\bx)$ with
    respect to $\|\cdot\|_{*}$. \\

    From local smoothness at $\bx^{*}$, we have for any $\by\in\calK$,
    \begin{align*}
        f(\by) =
        \frac{1}{2}\by^{\top} \nabla^2\phi^{*}(\bx^{*})\by
        \leq\frac{L}{2} \|\by\|^2_{*} ~.
    \end{align*}
    Taking the dual, which is order reversing, we have for any $\bz\in\calK$,
    \begin{equation}
        \label{dual_reverse:eq}
        f^{*}(\bz) =
        \frac{1}{2}\bz^{\top} [\nabla^2\phi^{*}(\bx^{*})]^{-1}\bz \geq
        \frac{1}{2L} \|\bz\|^2 ~.
    \end{equation}
    Since $\nabla\phi^* = (\nabla \phi)^{-1}$, then from the inverse function theorem,
    we have that
    \begin{align*}
        \nabla^2\phi^{*}(\bx^{*}) = [\nabla^2\phi(\bx)]^{-1}~.
    \end{align*}
    Using the above equality in \eqref{dual_reverse:eq}, we have for any $\bz\in\calK$,
    \begin{align*}
        \frac{1}{2}\bz^{\top} \nabla^2\phi(\bx)\bz \geq \frac{1}{2L} \|\bz\|^2~.
    \end{align*}
\end{proof}\vspace{-32pt}

\subsection{Strong Convexity of Spectral Hypentropy}
We now analyze the strong convexity of the spectral \hypent. The spectral
function $\Phi_{\beta}(\bX)$ is defined for $\bX\in\RR^{m \times n}$ by
\eqref{singular_func:eq} and for $\bX\in\SS^d$ by \eqref{singular_func_b:eq}
replacing $f$ with $\phi_\beta$. The main theorem of this subsection is as
follows.
\begin{theorem}
    \label{asymmetric_matrix_strcvx:theorem}
    The trace function $\Phi_{\beta}: \RR^{m \times n} \rightarrow \RR$ is
    $(2(\tau+ \beta \min\{m,n\}))^{-1}$-strongly convex with respect to the trace norm
    over $\btr{\tau}$.
\end{theorem}

We denote the $d$-dimensional symmetric matrices of trace-ball with maximal
radius $\tau$ by $$\btrs{\tau} = \{\bX \in \SS^d: \|\bX\|_1 \leq \tau\} ~.$$ We
prove the above theorem by first proving the lemma below for matrices in
$\btrs{\tau}$. We then extend it to arbitrary matrices using a symmetrization
argument, using a technique similar to \cite{juditsky2008large,
warmuth2007winnowing, hazan2012near}.

\begin{definition}
    \label{sym_rank:def}
    Let $\calK \subseteq \btrs{\tau}$ be a subset of matrices and $\calX \subseteq \SS^d$
    be a vector space containing $\calK$ such that
    $\forall \bX \in \calX, \rank(\bX) \leq r$ and $\nabla\Phi_{\beta}(\bX) \in \calX$.

\end{definition}

This abstraction will be useful in translating strong convexity of arbitrary matrices
to the symmetric case. The bound on the rank is essential to give a modulus of
strong convexity result that depends only on $\min\{m, n\}$ rather than $m+n$.
The final property is necessary for the low rank structure to be preserved after
a primal-dual mapping.

\begin{lemma}
    \label{symmetric_matrix_strcvx:theorem}
    The trace function $\Phi_{\beta}$ is $(2(\tau + \beta r))^{-1}$-strongly
    convex w.r.t the trace norm over $\calK$.
\end{lemma}

To prove the symmetric variant, we show that $\Phi^{*}_{\beta}$ is smooth
with respect to the spectral norm, which is the dual norm of the trace norm.
The result then follows directly from Lemma~\ref{local_dual_smooth:lemma}.

From Theorem~\ref{matrix_conj:theorem}, $\Phi_{\beta}$ has Fenchel conjugate
\begin{align*}
    \Phi^{*}_{\beta}(\bX) = \Tr(\phi^{*}_{\beta}(\bX))~.
\end{align*}
Since the derivative of the conjugate of a function is the inverse of the derivative
of the function, we have
$$
\frac{d\phi^{*}_{\beta}}{dx} =
\left(\frac{d\phi_{\beta}}{dx}\right)^{-1}= \beta\sinh(x) ~.
$$
The indefinite integral of the above yields that up to a constant,
$\phi^{*}_{\beta}(x) = \beta \cosh(x)$. Clearly, $\Phi_{\beta}$ is not smooth
everywhere. Nonetheless, we do have smoothness over $\nabla\Phi_{\beta}(\btrs{\tau})$.
Before proving this property, we introduce a clever technical lemma
of~\citet{juditsky2008large} that allows us to reduce the spectral smoothness
for matrices to smoothness of functions in the vector-case.
\begin{lemma}% \cite{juditsky2008large}
    \label{seconddirderiv:lemma}
    Let $f:\RR \rightarrow \RR$ be a function and $c\in\RR_+$ such that
    that for $a \geq b$,
    \begin{align}
        \frac{f'(a) -f'(b)}{a-b} \leq \frac{c(f''(a) + f''(b))}{2}~.
    \end{align}
    Let $F: \SS^d \rightarrow \RR$ be a function defined by $F(\bX)= \Tr(f(\bX))$.
    Then, the second directional derivative of $F$ is bounded for any $\bH\in\SS^d$ as
    follows,
    \begin{align*}
        \dd{\bH}{F(\bX)} \leq c\Tr(\bH f''(\bX)\bH)~.
        % \nabla^2_\bH\left[F(\bX)\right] \leq c\Tr(\bH f''(\bX)\bH)~.
    \end{align*}
\end{lemma}

We are now prepared to analyze the smoothness of $\Phi^{*}_{\beta}$.
\begin{lemma}[Local Smoothness] % of $\Phi^{*}_{\beta}$]
    The trace function
    $\Phi^{*}_{\beta}$ is locally $2(\tau+ \beta r)$-smooth with respect to the spectral
    norm for all matrices in $\nabla\Phi_{\beta}(\calK)$.
\end{lemma}

\begin{proof}
    To prove local smoothness, we use the second order conditions from Lemma
    \ref{second_strsmooth:lemma}. This requires us to upper bound the second
    directional derivatives for all directions corresponding to matrices of unit
    spectral norm. We consider the matrix $$\bX = \nabla\Phi(\bY) =
    \phi_{\beta}'(\bY) = \arcsinh\left(\frac{\bY}{\beta}\right) ~.$$ Note that
    $(\phi^{*}_{\beta})''(x) = \beta\cosh(x)$ is positive and convex.
    Therefore, by the mean value theorem, there exists $c \in [a,b]$ for which,
    \begin{align*}
        \frac{(\phi^{*}_{\beta})'(b) - (\phi^{*}_{\beta})'(a)}{b-a} =
        (\phi^{*}_{\beta})''(c) \leq
        \max\left\{\phi^{*}_{\beta})''(a) \,,\, (\phi^{*}_{\beta})''(b)\right\} \leq
        (\phi^{*}_{\beta})''(a) + (\phi^{*}_{\beta})''(b) ~.
    \end{align*}

    We note that by Definition~\ref{sym_rank:def}, $\nabla\Phi_{\beta}(\calX)
    \subseteq \calX$, so we can restrict ourselves to the vector space $\calX$.
    Therefore, applying Lemma \ref{seconddirderiv:lemma}, we have
    \begin{align*}
        \sup_{\bH\in \calX: \|\bH\|_{\infty} \leq 1}\dd{\bH}{\Phi_{\beta}^{*}(\bX)}
        &\leq \sup_{\bH\in \calX: \|\bH\|_{\infty}\leq 1} 2 \Tr\left(\bH (\phi^{*}_{\beta})''(\bX)\,\bH\right)\\
        &= \sup_{\bH\in \calX: \|\bH\|_{\infty}\leq 1} 2\Tr\left(\bH^2 (\phi^{*}_{\beta})''(\bX)\right)
        & [\mbox{Commutativity of trace}]\\
        &\leq \sup_{\bH\in \calX: \|\bH\|_{\infty}\leq 1} 2\ip{\sigma^2(\bH)}{\sigma((\phi^{*}_{\beta})''(\bX))}
        & [\mbox{von Neumann's trace inequality}]
    \end{align*}
    where von Neumann's trace inequality stands for,
    $\Tr(\bA^{\top}\bB) \leq \ip{\sigma(\bA)}{\sigma(\bB)}$.
    Now, since $\bH\in \calX$, we know $\rank(\bH)\leq r$, and so $\bH$ can have
    at most $r$ nonzero singular values, yielding
    \begin{align*}
        \sup_{\bH\in \calX: \|\bH\|_{\infty} \leq 1}\dd{\bH}{\Phi_{\beta}^{*}(\bX)}
        &\leq \sup_{\bH\in \calX:\|\bH\|_{\infty} \leq 1}
        2\|\bH^2\|_{\infty}\sum_{i=1}^r\sigma_i((\phi^{*}_{\beta})''(\bX)) \\
        &= 2\sum_{i=1}^r(\phi^{*}_{\beta})''(\phi_{\beta}'(\sigma_i(\bY))) ~.
    \end{align*}
    % 			& [\mbox{H\"older's inequality}] \\ &=
    % 			2\Tr\left((\phi_{\beta}^{*})''(\phi_{\beta}'(\bY))\right) ~,
    Now, note that $$(\phi_{\beta}^{*})''(\phi_{\beta}'(x)) =
    \beta\cosh\left(\arcsinh\left({x}/{\beta}\right)\right) = \sqrt{\beta^2 +
    x^2} \leq \beta + |x| ~ .$$ It then follows that \begin{align*}
                                                         \sum_{i=1}^r(\phi^{*}_{\beta})''(\phi_{\beta}'(\sigma_i(\bY))) \leq \beta
                                                         r + \|\bY\|_1 \leq \tau+\beta r~.
    \end{align*}
    % \begin{align*} \Tr\left((\phi_{\beta}^{*})''(\phi_{\beta}'(\bY))\right)
    % \leq \beta d + \|\bY\|_1 \leq 1+\beta d~.  \end{align*}
    Therefore, the second directional derivative in bounded by $2(\tau+\beta r)$
    as desired.
\end{proof}

\iffalse
\ys{Not clear what you mean by the following.\ldots "H\"older's inequality for
matrices can be used to combine this step and the subsequent step, we split
this for later use."} \ug{I previously used a label in the von Neumann's trace
inequality step, as in the asymmetric case, we have at most $\min\{m, n\}$
nonzero singular values. If we didn't need this refinement, using the Von
Neumman trace inequality would be unnecessary, as we could replace this
inequality and H\"older's inequality with the generalization of H\"older's
inequality for matrices in a single step. In the current state there is a
dangling reference in the proof of
Theorem~\ref{asymmetric_matrix_strcvx:theorem}. I wasn't super happy with
this anyway, as this theorem isn't truly being used as a lemma. In the general
theorem, I subtly refined the analysis. I have changed the the lemma to apply
over a subset of the symmetric trace-ball $\calK \subseteq \btrs{1}$ and
provided a bound that depends on a rank bound for $\calK$. I can easily
recover the previous version using git, so I'm avoiding the iffalses for this
change.}
\fi

\begin{proof}[Theorem~\ref{asymmetric_matrix_strcvx:theorem}]
    We introduce the symmetrization operator $S:\RR^{m\times n} \rightarrow \SS^{m+n}$
    which is a linear function that takes a matrix to a symmetric matrix,
    \begin{align*}
        S(\bX) = \begin{bmatrix}
                     0 & \bX \\
                     \bX^{\top}& 0
        \end{bmatrix}~.
    \end{align*}
    The eigenvalues of $S(\bX)$ are exactly one copy of singular values
    and one copy of negative singular values of $\bX$. Therefore, we have
    $$\Phi_{\beta}(\bX) =
    \sum_{i=1}^{\min\{m,n\}}\hspace{-6pt}\phi_{\beta}(\sigma_i(\bX)) =
    \frac12{\Phi_{\beta}(S(\bX))} ~. $$
    Technically, for the above to hold true, we should shift $\phi_{\beta}$ such
    that its $0$ is at $0$. Since a constant shift does not affect
    diameter or convexity properties so this is not an issue.

    \smallskip

    Let $\mu$ be the modulus of strong convexity in the symmetrized space. We
    bound $\Phi_{\beta}(\bX)$ from below as follows,
    \begin{align*}
        2\Phi_{\beta}(\bX) &= \Phi_{\beta}(S(\bX)) \\
        &\geq \Phi_{\beta}(S(\bY)) + \ip{\nabla \Phi_{\beta}(S(\bY))}{S(\bX) - S(\bY)} + \frac{\mu}{2} \|S(\bX) -S(\bY)\|^2_1\\
        &= \Phi_{\beta}(S(\bY)) + \ip{\nabla \Phi_{\beta}(S(\bY))}{S(\bX -\bY)} + \frac{\mu}{2} \|S(\bX -\bY)\|^2_1\\
        &= 2\Phi_{\beta}(\bY) + 2\ip{\nabla \Phi_{\beta}(\bY)}{\bX -\bY} + 2\mu \|\bX -\bY\|^2_1 ~ .
    \end{align*}
    Therefore, the modulus of strong convexity over asymmetric matrices is $2\mu$.

    Note that $\calK = \{S(\bX) : \bX \in  \btr{2\tau}\}$ satisfies the properties listed
    in Definition~\ref{sym_rank:def}, with $r= 2\min\{m,n\}$. In particular,
    we have a vector space $\calX = \{S(\bX): \bX \in \RR^{m \times n}\}$ containing
    symmetric matrices of rank at most $2\min\{m,n\}$. Furthermore, for any
    $\bX \in \calX$, $\bX = S(\bY)$ for some $\bY\in \RR^{m\times n}$
    and thus,
    \begin{align*}
        \nabla\Phi_{\beta}(\bX) = \nabla\Phi_{\beta}(S(\bY))=
        S(\nabla\Phi_{\beta}(\bY)) \in \calX ~.
    \end{align*}
    It follows from Theorem~\ref{symmetric_matrix_strcvx:theorem} that
    $\mu \geq (2(2\tau+ 2\beta \min \{m,n\}))^{-1}$,
    and so the strong convexity is at most $2\mu \geq (2(\tau+ \beta \min \{m,n\}))^{-1}$
    as desired.
\end{proof}\vspace{-16pt}

\subsection{SHU algorithm}
We next describe, the Spectral Hypentropy Update ($\SHU$),
an OCO algorithm over a convex domain of matrices $\calK \subseteq\RR^{m \times n}$.

%%%%%%%%%%%%%%%%%%%%%%%%%%%%%%%%%%%%%%%%%%%%%%%%%%%%%%%%%%%%%%%%%%%%%%%%%%%%%%%
\begin{algorithm2e}[ht]
    \label{shu:algorithm}
    \SetAlgoLined
    \SetKw{KwBy}{by}
    \KwIn{$\eta >0, \beta > 0$, convex domain of matrices $\calK \subseteq \RR^{m \times n}$}
    Initialize weight matrix $\bW^{1} = \bzero$\;
    \For{$i=1$ \KwTo $T$}{
    (a) Predict $\bW^{t}$ ~ ~
    (b) Incur loss $\ell_t(\bW_t)$ ~ ~
    (c) Calculate $\bG^t =  \nabla \ell_t(\bW^t)$ \;
    \smallskip
    Update:
    % $ \begin{align}\label{SHU_update:eq}
    $  \bW^{t+\frac{1}{2}} =
    \beta \sinh{\Big(\arcsinh{\Big(\frac{\bW^{t}}{\beta}\Big)} - \eta
    \bG^{t}}\Big) $\;
    % \end{align}
    \smallskip
    Project onto $\calK$:
    % \begin{align*}
    $\bW^{t+1} = \displaystyle
    \argmin_{\bV \in \calK}
    D^{\beta}_{\Phi}\infdivx{\bV}{\bW^{t+\frac{1}{2}}}$
    % \end{align*}
    }
    \caption{Spectral \Hypent Update (SHU)}
\end{algorithm2e}
%%%%%%%%%%%%%%%%%%%%%%%%%%%%%%%%%%%%%%%%%%%%%%%%%%%%%%%%%%%%%%%%%%%%%%%%%%%%%%%

The pseudocode of $\SHU$ is provided in
Algorithm~\ref{shu:algorithm}.  The update
step of SHU requires a spectral decomposition. We define
$f(\bA) = \bU f(\diag\left[\sigma(\bA)\right])\bV^{\top}$ where
$\bA = \bU\diag(\sigma(\bA))\bV^{\top}$ is the singular value decomposition of
$\bA$. We use this definition twice, once with $f=\arcsinh$, and after
subtracting the gradient with $f=\sinh$. We prove the following regret
bound for $\SHU$.
\begin{theorem}
    \label{trace_normregret:theorem}
    Let $\bW \in \btr{\tau} \subseteq \RR^{m\times n}$ and let
    $\|\bG^t\|_{\infty} \leq G_{\infty}$ be a spectral norm bound on the gradients.
    For $\gamma = \frac{\beta}{\tau} \leq 1$, setting,
    $$\eta =
    \frac{1}{2G_{\infty}}
    \sqrt{\frac{ \log{\big(\tfrac{3}{\gamma}\big)}}{T(1+\gamma\min\{m,n\})}}
    ~,~ \mbox{ yields } ~ ~
    \regret_T(\SHU) \leq
    4\tau G_{\infty} \sqrt{T (1+\gamma\min\{m,n\})\log{(\tfrac{3}{\gamma})}}~.
    $$
\end{theorem}

\begin{proof}
    Like for $\HU$, we use the general OMD analysis.
    It suffices to find an upper bound on $\diam_{\Phi_{\beta}}(\btr{\tau})$ and a strong convexity bound.

    Applying~\eqref{l1_diameter:eq} on the vector singular values, we have
    $$\diam_{\Phi_{\beta}}(\btr{\tau}) \leq \tau \log\left(\frac{3\tau}{\beta}\right)~,$$
    where $\beta \leq \tau$.
    Furthermore, from Theorem~\ref{asymmetric_matrix_strcvx:theorem}, we can see that
    $\Phi_{\beta}$ is $(2(\tau+ \beta \min\{m,n\}))^{-1}$
    strongly convex with respect to the trace-norm over $\btr{\tau}$.
    Letting $\gamma = \frac{\beta}{\tau}$, the result follows from Theorem~\ref{omd:theorem}.
\end{proof}

%% file: experiments.tex
\section{Experimental Results} \label{experiments:sec}
Next, we experiment with HU in the context of empirical risk minimization
(ERM). In the experiments, $\bg_t$ stands for a stochastic estimate of the
gradient of the empirical loss. Thus, we can convert the regret analysis to
convergence in expectation~\cite{ONLINEtoSTOCHASTIC}.

\paragraph{Effective Learning Rate}
For small value $w$, $\sinh(w) \approx w \approx \arcsinh(w)$. As a result, near $0$
the update in $\HU$ is morally the additive update, $w_i^{t+1}
= w_i^t -\beta\eta g^{t}_i$. The product $\beta\eta$ can be viewed as the de
facto learning rate of the gradient descent portion of the interpolation. As
such, we define the \textit{effective learning rate} to be $\beta\eta$. In the
sequel, fixing the effective learning rate while changing $\beta$ is a
fruitful lens for comparing $\HU$ to with $\GD$.

\subsection{Logistic Regression}
In this experiment we use the $\HU$ algorithm to optimize a logit model.
The ambient dimension $d$ is chosen to be $500$. A weight $\bw$
is drawn uniformly at random from $[-1,1]^d$. The features are from
$\{0, 1\}^d$ and distributed according to the power law,
$\Pr[x_i =1] = {1}/{5\sqrt{i}}$. The label associated with an example
$\bx_t$ is set to $y_t = \sign(\ip{\bw}{\bx_t})$ with with probability
$0.9$ and otherwise flipped.

The algorithms are trained with log-loss using batches of size $10$.
Stochastic gradient descent and the $p$-norm algorithm~\cite{Gentile2003}
are used for comparison. As can be seen in Fig.~\ref{synth_lin:fig},
the $p$-norm algorithm performs significantly worse than HU for a large set of values of $\beta$, while
SGD performs comparably. As expected, for large value of $\beta$, SGD and
HU are indistinguishable.

\begin{figure}[h]
    \centerline{\hbox{
    \includegraphics[width=.4\linewidth]{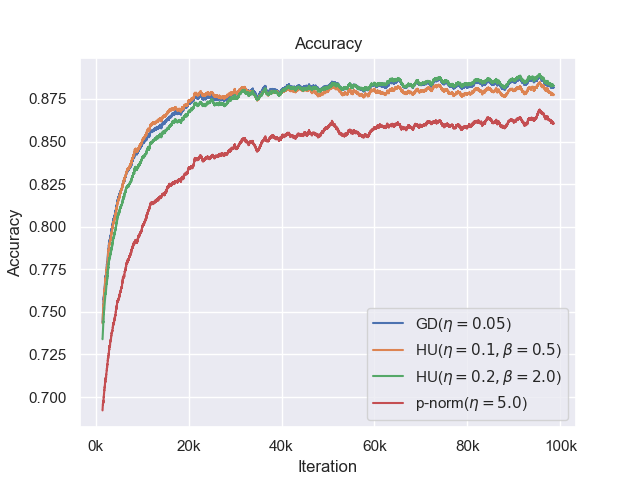}\quad
    \includegraphics[width=.4\linewidth]{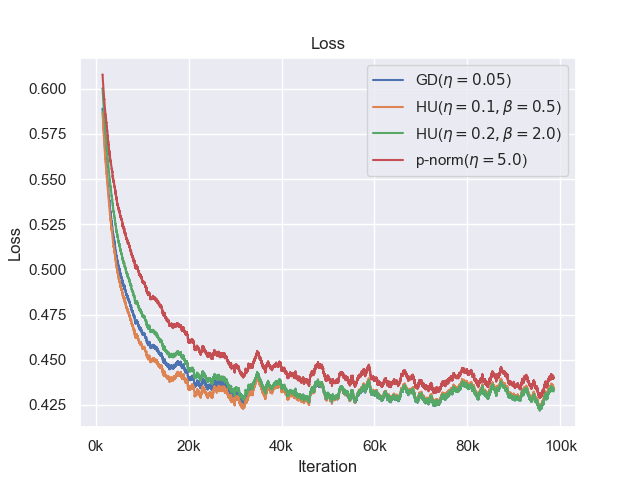}}}
    \vspace{-12pt}
    \caption{Comparison of accuracy and loss of GD, $p$-norm, and $\HU$ on binary
    logistic regression.}
    \label{synth_lin:fig}
\end{figure}
In the next experiment we use the same logit model with ambient
dimension $d$ chosen to be $10,000$. We generate weights, $\bw \in \RR^d$,
with sparsity (fraction of zero weights) $s \in\{0, 0.9\}$. The nonzero
weights are chosen uniformly at random from $[-1, 1]$. We run the algorithms
for 20,000 iterations. Rather than fixing $\eta$, we fix
$\eta' = {\eta}/{\sqrt{1 + \beta^2}}$. This way, as $\beta\to\infty$,
$\HU(\eta', \beta)$ behaves like $\GD(\eta)$ while for
small $\beta$, the update is roughly $\EGpm(\eta)$. We let $\beta_{\EG} =
{\|\bw^{\star}\|_1}/{d}$. As discussed in Appendix~\ref{egpm_appendix:sec}, this
choice of $\beta$ is similar to running $\EGpm$ with an $1$-norm bound of
$\|\bw^{\star}\|_1$. We then choose $\eta' = 0.1$ and
$\beta \in \{0.5, 1, 2, 4, 8\}\times\beta_{\EG}$.
In Fig.~\ref{synth_lin_v2:fig}, we show the
interpolation between $\GD$ and $\EGpm$. The larger $\beta$ is, the closer the
progress of HU resembles that of $\GD$. Intermediate values of $\beta$ have
progress in between the $\EGpm$ and $\GD$.
\begin{figure}[h]
    \centerline{
    \includegraphics[width=.4\linewidth]{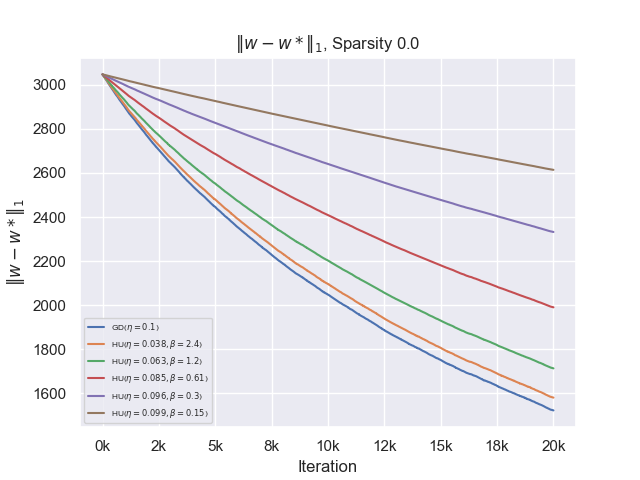}\quad
    \includegraphics[width=.4\linewidth]{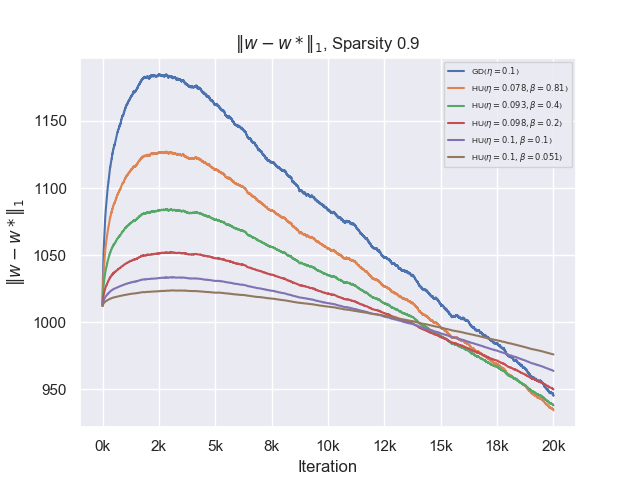}
    }
    \caption{Value of $\|\bw^t - \bw^{\star}\|_1$ in dense and sparse settings
    with $\eta' = 0.1$ and $\beta \in \{0.5, 1, 2, 4, 8\}\times\beta_{\EG}$.}
    \label{synth_lin_v2:fig}
\end{figure}

\subsection{Multiclass Logistic Regression}
In this experiment we use $\SHU$ to optimize a multiclass
logistic model. We generated 200,000 examples in $\RR^{25}$. We set the number
of classes to $k=15$. Labels are generated using a rank $5$ matrix
$\bW \in \RR^{k \times d}$. An example $\bx \in \RR^d$ was labeled according to
the prediction rule,
\begin{align*}
    y= \argmax_{i \in [k]} \{(\bW \bx)_i\} ~.
\end{align*}

With probability $5.0\%$ the label was flipped to a different one at random.
The matrix $\bW$ and each example $\bx_i$ features are determined in a joint
process to make the problem poorly conditioned for optimization. Features of
each example are first drawn from a standard normal. Weights of $\bW$ are
sampled from a standard normal distribution for the first $r$ features and
are set to $0$ for the remaining $d-r$ features. After labels are computed,
features are perturbed by Gaussian noise with standard deviation $0.05$. The
examples and weights are then scaled and rotated. Coordinate $i$ of the data
is scaled by  $s_i \propto i^{-1.1}$ where $\sum_{i=1}^d s_i = 1$. Then a
random rotation $\bR$ is applied. The inverse of these transformation is
applied to the weights. Therefore, from the original sample
$\bX_0\in\RR^{n\times d}$ and weights $\bW_0 \in \RR^{k \times d}$ the new
sample and weights are set to be $\bX = \bX_0\bR$ and $\bW= \bW_0\bR^{-1}$,
where $\bR \in \RR^{d \times d}$ is the scaling and rotation described above.

Since our ground truth weights are low rank, our goal is to find
weights of approximately low rank with low classification error. To do
this, we optimize a multiclass logistic regression loss with a
trace-norm constraint. We compare $\SHU$,
Schatten $p$-norm algorithm ($p$-norm algorithm applied to singular values) and
gradient descent in the fully stochastic (single example) case. We report results
for unconstrained optimization in Fig.~\ref{multiclass_unconstrained:fig} and
trace-norm constrained optimization in Fig.~\ref{multiclass_constrained:fig}.
In these figures only the algorithms with lowest final loss after a
grid search are depicted.

\begin{figure}[h]
    \includegraphics[width=.35\linewidth]{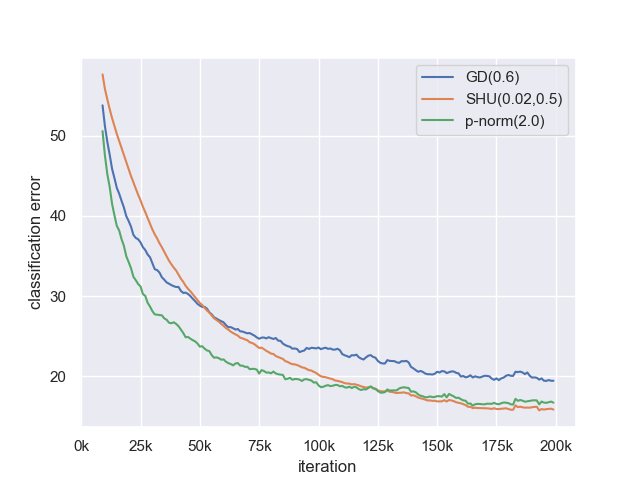}\quad
    \includegraphics[width=.35\linewidth]{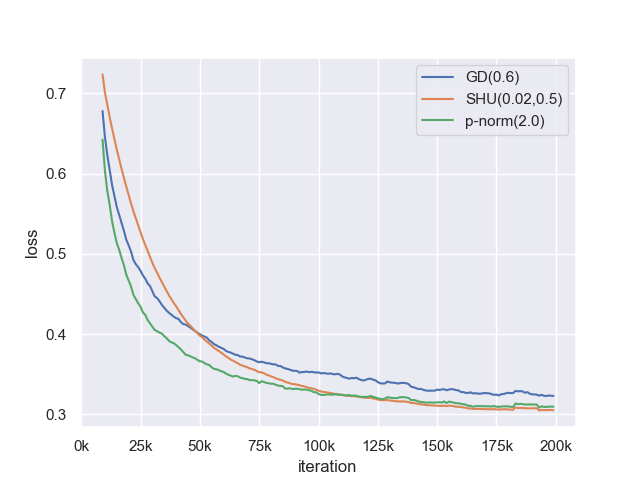}\quad
    \includegraphics[width=.25\linewidth]{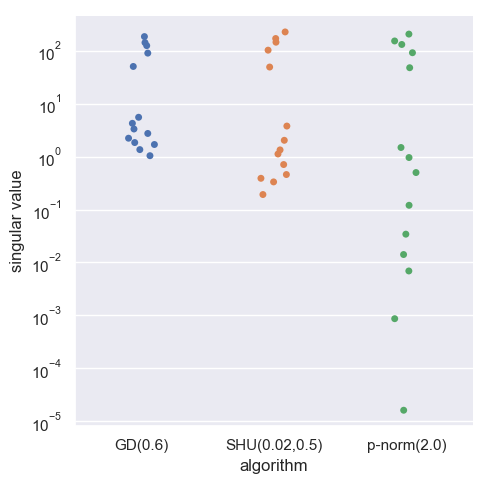}
    \caption{Unconstrained minimization of logistic loss.}
    \label{multiclass_unconstrained:fig}
\end{figure}
\begin{figure}[h]
    \includegraphics[width=.35\linewidth]{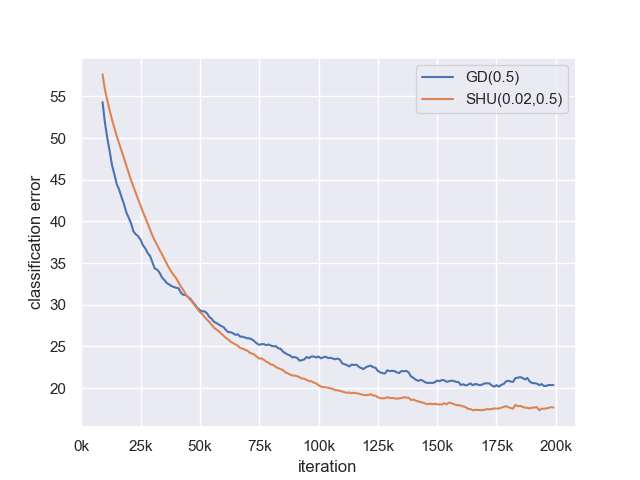}\quad
    \includegraphics[width=.35\linewidth]{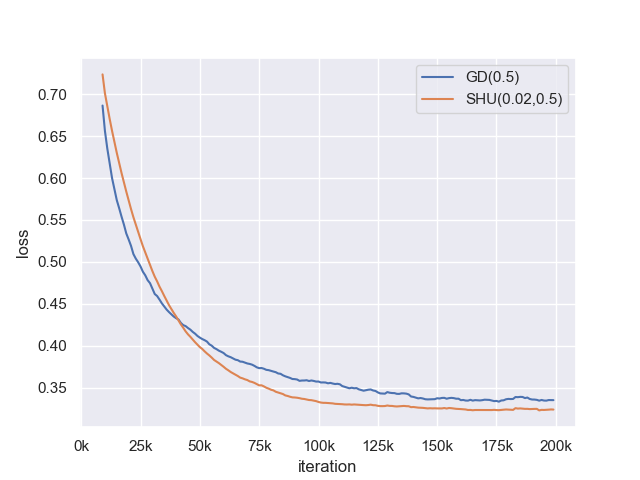}\quad
    \includegraphics[width=.25\linewidth]{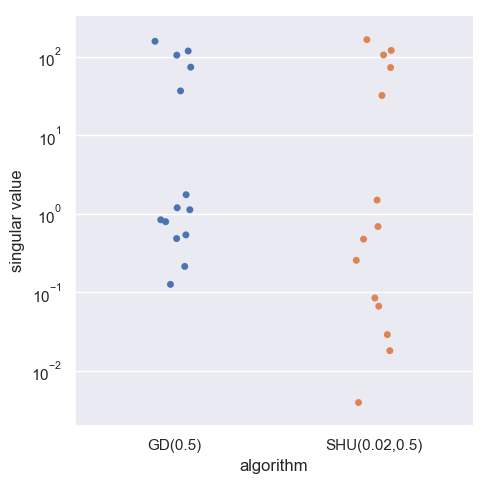}
    \caption{Minimization over the trace-norm ball of radius $500$. The $p$-norm algorithm
    is not included becaus ethe $p$-norm divergence does not have a closed form projection onto the $1$-ball.}
    \label{multiclass_constrained:fig}
\end{figure}

Without projection, $\SHU$ results in the largest trace norm
solution of norm $700$ whereas the $p$-norm algorithm and $\GD$ reach
solutions with trace norm just above $600$. Nevertheless, $\SHU$ attains the lowest
classification error and loss. Performance is noticeably better than gradient
descent. Moreover, the spurious singular values are typically smaller than that
of gradient descent. This pattern holds up in both settings.

With our new divergence, the $\SHU$ update looks exponential for
large singular values and linear for small ones. In this sense,
once gradients start accumulating in the directions that correspond
to the actual signal, these directions can be exploited exponentially.
On the other hand, the spurious directions are morally handled with
gradient descent with an effective learning rate $\eta\beta$. Note that
in our experiments, the $\SHU$ effective learning rate is smaller than
the $\GD$ learning rate by an order of magnitude. This may explain the
smaller magnitude in erroneous singular values. On the other hand,
the $p$-norm algorithm not only increases the magnitude of large
singular values, but also shrinks the magnitude of small singular values,
resulting in solutions that are closer to being low-rank. To see this,
note that the $p$-norm inverse mirror map has the form
\begin{align*}
    f(\bsigma)_i=\sigma_i^{p-1}/\|\bsigma\|^{p-2}_p
\end{align*}
for $p=2\ln(k)\approx 5.4$. Therefore, there is a natural normalization which
shrinks small singular values as good directions are exploited. Without
projection, this does not happen with $\SHU$. Informally speaking, the
following analogy applies to the three methods:
\begin{center}
    $\GD$: ~ The rich get richer! \\
    $\SHU$: ~ The rich get \emph{much} richer!! \\
    $p$-norm: ~ The rich get richer and the poor get poorer, oy!
\end{center}

Adding trace norm projection reduces
the magnitude of these singular values, but not to the level which the $p$-norm
algorithm can achieve. Overall, it appears that $\SHU$ may be
slightly more effective at reducing loss but the $p$-norm algorithm is more
effective at producing low rank solutions.

\subsection{Image Classification with Neural Networks.}
Loss minimization for neural networks is known to be nonconvex, thus the regret
bounds from this paper do not apply in this setting. Still, convex optimization
algorithms, such as AdaGrad, work well practice for training neural networks.
In this section, we use the unconstrained version of the HU to find the weights
of a simple neural network for image classification using the popular CIFAR10
dataset~\cite{CIFAR10}. SGD was used for comparison. Outside of use of the HU
algorithm, the design of the network and code are from the Tensorflow
tutorial on convolutional networks for image classification. The network
involves $2$ convolutional layers, max pooling, and $2$ fully connected layers,
all using ReLU activations. The loss function is the cross entropy loss plus a
$2$-norm regularization. For a complete description of the experimental setup
see~\cite{TFTUTORIAL, krizhevsky2009learning}.

Empirically, SGD with learning rate $\eta$ tended to perform similarly in terms
of training error to HU with equivalent effective learning rate $\beta\eta$ for
a range of values of $\beta$. In order to compare to SGD with learning rate
$\eta$, $\beta$ was varied and HU's learning rate was set to be
$\frac{\eta}{\beta}$ (in order to keep the effective learning rate invariant).
As can be seen from Figure~\ref{cifar10:fig}, the loss curves for a variety of
values of $\beta$ are very similar for $\beta\eta =0.005$, although the
smallest $\beta=0.1$ has slightly lower loss. In general, the final loss
reached is similar for a fixed effective learning rate. In addition, there is
a clear pattern indicating that shrinking $\beta$ results in sparser weights.
This may warrant further investigation.

\begin{figure}[h]
    \centerline{\hbox{
    \includegraphics[width=.4\linewidth]{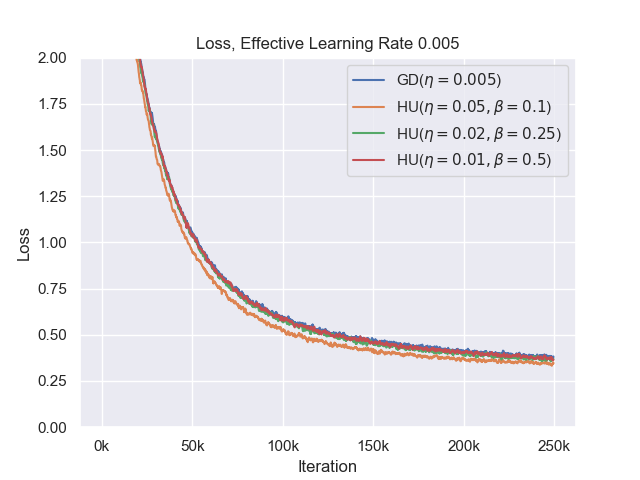}\quad
    \includegraphics[width=.4\linewidth]{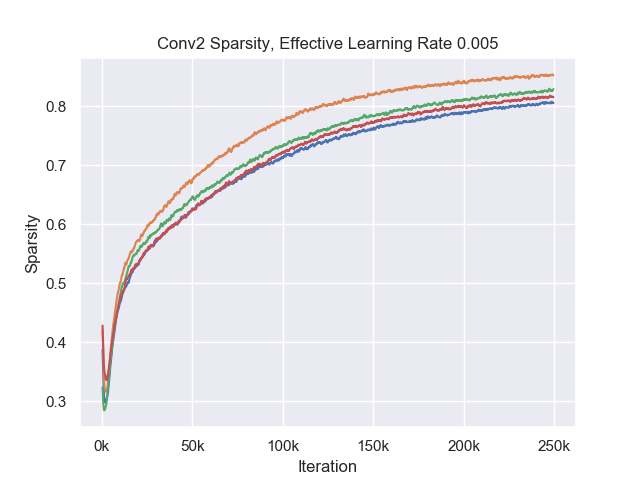}
    }}
    \caption{CIFAR10 loss and sparsity level. The effective learning rate was held
    constant at $0.005$ while $\beta$ is varied. Sparsity is displayed for the
    second convolutional layer conv2. The loss corresponds to the total loss,
    which includes regularization.}
    \label{cifar10:fig}
\end{figure}

%% file: discussion.tex
\section{Discussion}

\newcommand{\ignore}[1]{}

We examined a new regularization for online learning which interpolates between
multiplicative and additive updates through of a single parameter $\beta > 0$.
As $\beta \to \infty$, the algorithm approaches gradient descent while as
$\beta \to \frac{1}{d}$ it behaves similarly to the multiplicative update. The
spectral regularization provides a matrix analogue which is naturally
applicable to rectangular matrices. An interesting open direction is to devise
an self-tuning update for $\beta$ which is data dependent.

\ignore{
The first two theorems we describe provide bounds on the regret of the $\HU$
algorithm when costs are bounded in $\ell_2$ and $\ell_{\infty}$ norms
respectively. Theorem~\ref{l2regret:theorem} shows that when $\beta \geq 1$ the
regret is $O(G_2\sqrt{T})$, which is asymptotically the same as the regret of
$\GD$. When $\beta \rightarrow \infty$, the update converges exactly to that of
gradient descent, as expected. In the other extreme,
when $\beta = \Theta(\frac{1}{d})$, Thm.~\ref{l1regret:theorem} implies that
the regret is $O(G_{\infty}\sqrt{T\log{(d)}})$, which is asymptotically
equivalent to that of the multiplicative update.

To further distill these results, note that $\beta = \frac{1}{d}$ corresponds to
the $\EGpm$ algorithm with a different mechanism for reducing the size of
weights to satisfy an $1$-norm constraints. In particular, we can view $\EGpm$
as an \emph{adaptive} variant of the $\HU$ algorithm. In this interpretation,
instead of using hypentropy projection to reduce he size of weights, $\beta_t$
is decreased such that projections are not necessary. In this sense, $\EGpm$
behaves multiplicativly as gradients accumulate in certain directions, and
additively when gradients are noisier. In contrast, $\HU$ keep $\beta$ fixed
so the geometry of the hypentropy regularization remains intact. As an
interpolation between weight normalization and soft thresholding,
the hypentropy projection has the potential of explicitly producing sparse
weights. Analysis of the connections between $\EGpm$ and hypentropy projection
is given in App.~\ref{egpm_appendix:sec} and App.~\ref{proj:sec}.

Note that as $\beta$ shrinks towards $0$, the regret bounds eventually
deteriorate. With a sufficiently small $\beta$, a weight crossing $0$ can be
prohibitively expensive. In the $\EGpm$ setup, $\beta$ corresponds to
the size of an initial positive / negative weight. For $\beta$ arbitrarily
small, it may take arbitrarily many multiplicative updates in order to
reach reasonable weights.

Theorem~\ref{trace_normregret:theorem} provides a regret bound where the
decision set is a set of matrices with bounded trace norm. This constraint can
be used as a convex relaxation for low rank matrices~\cite{Fazel2002} We do
not provide regret bounds for matrix counterpart of the $2$-norm, namely,
matrices of bounded Frobenius norm, even though analogous results hold.
As the Frobenius norm is equivalent to the Euclidean norm of a flattened
$mn$ dimensional vector, the $\HU$ algorithm can be used verbatim
in this setting and is computationally less expensive.

We note that for rectangular matrices, Theorem~\ref{trace_normregret:theorem}
provides a regret bound that depends only on the minimum of $n$ and $m$. When
$\beta = \Theta({\tau}/{\min\{m,n\}})$, the regret is $O(\tau
G_{\infty}\sqrt{T\log(\min\{m,n\})})$.
}

%% file: omd_background.tex
\section{Online Mirrored Descent}\label{appendex1:sec}

\begin{algorithm2e}[h]
    \label{omd:algorithm}
    \SetAlgoLined
    \SetKw{KwBy}{by}
    \KwIn{$\eta >0, \beta > 0$, convex domain $\calK \subseteq \RR^{d}$}
    Let $\by_1$ be such that $\nabla R(\by_1) = 0$ and $\bw^{1} = \argmin_{\bw \in \calK} D_{R}\infdivx{\bw}{\by^1}$\;
    \For{$i=1$ \KwTo $T$}{
    (a) Predict $\bw^{t}$ ~ ~
    (b) Incur loss $\ell_t(\bW_t)$ ~ ~
    (c) Calculate $\bg^t =  \nabla \ell_t(\bw^t)$ \;
    \smallskip
    Update:
    $  \nabla R(\by^{t+1}) = \nabla R(\bw^t) - \eta\bg^t $\;
    \smallskip
    Project onto $\calK$:
    $\bw^{t+1} = \argmin_{\bw \in \calK} D_{R}\infdivx{\bw}{\by^{t+1}}$
    }
    \caption{Online Mirror Descent with potential function $R$}
\end{algorithm2e}

Online Mirror Descent (OMD) is an meta-algorithm for online convex
optimization. The regularization function, $R$, is assumed to be strongly convex,
smooth, and twice differentiable. Like GD, Mirror Descent is an iterative
algorithm involving a simple gradient update. $R$ defines a
mapping into a dual space where the updates occur, followed by an inverse
mapping to the original space. This step may result in a vector outside of
$\calK$, so a projection is required. An alternative formulation where the
regularization casts a trade-off between moving along the gradient direction
and staying close to the current iterate is,
\begin{align}
    \bw^{t+1} &= \argmin_{\bw \in \calK}\big\{\eta\ip{\bg^t}{\bw} +
    D_{R}\infdivx{\bw}{\bw^{t}}\big\} ~.
\end{align}

For Algorithm~\ref{omd:algorithm} with the above assumptions,
we have the following regret bound.
\begin{theorem}[OMD Regret]
    \label{omd:theorem_general}
    Assume $R$ is $\mu$-strongly convex with respect to a norm $\|\cdot\|$ whose
    dual is $\|\cdot\|_{*}$, then running \textup{OMD} with a fixed learning rate
    $\eta$ yields the following regret bound,
    \begin{align*}
        \regret_T \leq
        \frac{1}{\eta}\,{\sup_{\bw \in \calK} D_{R}\infdivx{\bw}{\bw^{1}}}
        + \frac{\eta}{2\mu} \sum_{t=1}^T \|\bg^t\|^2_{*} ~.
    \end{align*}
\end{theorem}

If we have a bound on the dual norm of a gradient, we can choose a learning
rate which minimizes the upper bound, yielding Theorem~\ref{omd:theorem}. We
next introduce two well known properties of Bregman divergences without proof.
The first technical lemma is the Bregman divergence analogue of the law of
cosines.
\begin{lemma}[Three-point Lemma]
    \label{3pt:lemma}
    For every three vectors $\bx, \by, \bz$,
    \begin{align*}
        D_{R}\infdivx{\bx}{\bz} =
        D_{R}\infdivx{\bx}{\by} +
        D_{R}\infdivx{\by}{\bz} -
        \ip{\nabla R(\bz) - \nabla R(\by)}{\bx -\by} ~.
    \end{align*}
\end{lemma}
The next lemma is an analogue of the Pythagorean theorem for Bregman
projections.
\begin{lemma}[Generalized Pythagorean Theorem]
    \label{pythog:lemma}
    Let
    $$\bx'= \Pi_{\calK, \phi}(\bx) =
    \argmin_{\by \in \calK} D_{R}\infdivx{\by}{\bx} ~, $$
    then
    $$ D_{R}\infdivx{\bz}{\bx} \geq
    D_{R}\infdivx{\bz}{\bx'} + D_{R}\infdivx{\bx'}{\bx}~.$$
\end{lemma}

\begin{proof}[Theorem~\ref{omd:theorem_general}]
    Let $\bw^{*} = \argmin_{\bw\in \calK}\sum_{t=1}^T\ell_t(\bw)$ be
    the best fixed predictor in hindsight, then
    \begin{align*}
        \ell_t(\bw^t) - \ell_t(\bw^{*}) &\leq \ip{\bg^{t}}{\bw^t -\bw^{*}} & [\mbox{Convexity}] \nonumber\\
        &=\frac{1}{\eta}\ip{ \nabla R(\bw^t) - \nabla R(\by^{t+1})}{\bw^{t} -\bw^{*}}  \nonumber\\
        &=\frac{1}{\eta}\ip{  \nabla R(\by^{t+1})- \nabla R(\bw^t)}{\bw^{*}- \bw^{t}} \nonumber\\
        &=\frac{1}{\eta}\big ( D_{R}\infdivx{\bw^{*}}{\bw^t} + D_{R}\infdivx{\bw^t}{\by^{t+1}} - D_{R}\infdivx{\bw^{*}}{\by^{t+1}} \big)
        & [\mbox{Lemma}~ \ref{3pt:lemma}] \nonumber\\
        &=\frac{1}{\eta}\big ( D_{R}\infdivx{\bw^{*}}{\bw^t} + D_{R}\infdivx{\bw^t}{\by^{t+1}}
        - D_{R}\infdivx{\bw^{*}}{\bw^{t+1}} - D_{R}\infdivx{\bw^{t+1}}{\by^{t+1}}  \big )
        & [\mbox{Lemma}~ \ref{pythog:lemma}] \nonumber\\
        &=\frac{1}{\eta}\big ( D_{R}\infdivx{\bw^{*}}{\bw^t} - D_{R}\infdivx{\bw^{*}}{\bw^{t+1}}) +\frac{1}{\eta}\big ( D_{R}\infdivx{\bw^t}{\by^{t+1}}
        - D_{R}\infdivx{\bw^{t+1}}{\by^{t+1}}  \big )
        ~.
    \end{align*}
    Note that the left hand side term telescopes when summing over $t$, yielding
    \begin{align}
        \label{md_regret_sum:eq}
        \sum_{t=1}^T \ell_t(\bw^t) - \ell_t(\bw^{*}) \leq
        \frac{D_{R}\infdivx{\bw^{*}}{\bw^1}}{\eta} +
        \frac{1}{\eta}\sum_{t=1}^T \Big(D_{R}\infdivx{\bw^t}{\by^{t+1}}  - D_{R}\infdivx{\bw^{t+1}}{\by^{t+1}}\Big)~.
    \end{align}
    If suffices to upper bound $D_{R}\infdivx{\bw^t}{\by^{t+1}} - D_{R}\infdivx{\bw^{t+1}}{\by^{t+1}}$.
    We start by substituting the definition of the Bregman divergence in $D_R$,
    \begin{align*}
        D_{R}\infdivx{\bw^t}{\by^{t+1}}  &- D_{R}\infdivx{\bw^{t+1}}{\by^{t+1}}\\
        &= R(\bw^t) -R(\bw^{t+1}) - \ip{\nabla R(\by^{t+1})}{\bw^t-\bw^{t+1}}\\
        &\leq\ip{\nabla R(\bw^{t})}{\bw^t-\bw^{t+1}} - \frac{\mu}{2}\|\bw^t - \bw^{t+1}\|^2 - \ip{\nabla R(\by^{t+1})}{\bw^t-\bw^{t+1}} &[\mu\mbox{-strong convexity}]\\
        &= \ip{\nabla R(\bw^{t}) -\nabla R(\by^{t+1})}{\bw^t-\bw^{t+1}} - \frac{\mu}{2}\|\bw^t - \bw^{t+1}\|^2\\
        &= \eta\,\ip{\bg^t}{\bw^t-\bw^{t+1}} - \frac{\mu}{2}\|\bw^t - \bw^{t+1}\|^2 &[\mbox{Update rule}]\\
        &\leq \eta\,\|\bg^t\|_{*}\,\|\bw^t-\bw^{t+1}\| - \frac{\mu}{2}\|\bw^t - \bw^{t+1}\|^2 &[\mbox{ Cauchy-Schwarz}]\\
        &\leq \frac{\eta^2\|\bg^t\|^2_{*}}{2\mu} ~.
    \end{align*}
    The last step follows from maximizing the quadratic function in $\|\bw^t - \bw^{t+1}\|$.
    Using the above bound \eqref{md_regret_sum:eq} completes the proof.
\end{proof}

%% file: egpm.tex
\section{Connections to EG$\pm$} \label{egpm_appendix:sec}
The $\EGpm$ algorithm \cite{kivinen1997exponentiated} maintains two
vectors $\bu$ and $\bv$ such that, $\bw=\bu-\bv$. The two vectors
are updated over $\RR^{2d}_+$ using the $\EG$ algorithm. We consider here
a variant where the $2d$ dimensional weights are normalized such that
their sum is $\beta d$. Typically, we would have $\beta = \frac{1}{2d}$,
so $(\bu, \bv)$ lie on a unit simplex.

%%%%%%%%%%%%%%%%%%%%%%%%%%%%%%%%%%%%%%%%%%%%%%%%%%%%%%%%%%%%%%%%%%%%%%%%%%%%%%%
\begin{algorithm2e}[t]
    \label{egpm:algorithm}
    \SetAlgoLined
    \SetKw{KwBy}{by}
    \KwIn{$\eta >0, \beta > 0$}
    Initialize: $\bu^{1}_i= \bv^{1}_i = \frac{\beta}{2}$,~ $\barg^0 = \bzero$\;
    \For{$t=1$ \KwTo $T$}{
    (a) Predict $\bw^{t}=\bu^{t} - \bv^{t}$ ~ ~
    (b) Incur loss $\ell_t(\bw_t)$ ~ ~
    (c) Calculate $\barg^t =  \barg^{t-1} + \nabla \ell_t(\bw^t)$ \;
    \smallskip
    Update:
    % $ \begin{align}\label{SHU_update:eq}
    $  u^{t+\frac{1}{2}}_i = u^{t}_i \exp(-\eta g^{t}_i)$
    ~ ~and~ ~
    $ v^{t+\frac{1}{2}}_i = v^{t}_i \exp(\eta g^{t}_i)$\;
    % \end{align}
    Normalize weights:
    % \begin{align*}
    $
    (\bu^{t+1},\bv^{t+1}) =
    {\beta d} \,
    \Bigg({\sum_{i=1}^d u^{t+\frac{1}{2}}_i + v^{t+\frac{1}{2}}_i}\Bigg)^{-1}
    \!\! (\bu^{t+\frac{1}{2}},\bv^{t+\frac{1}{2}})
    $
    % \end{align*}
    }
    \caption{$\EGpm$}
\end{algorithm2e}
%%%%%%%%%%%%%%%%%%%%%%%%%%%%%%%%%%%%%%%%%%%%%%%%%%%%%%%%%%%%%%%%%%%%%%%%%%%%%%%

We now show that the $\EGpm$ algorithm can be viewed as an adaptive variant of $\HU$
with the update,
\begin{align}
    \label{ahu:eq}
    \by^{t+1} = (\nabla \phi_{\beta_t})^{-1}(\nabla \phi_{\beta_{t-1}}(\bw^{t}) - \eta \bg^{t}) ~.
\end{align}

The weight, $\bw^{t+1}$ is then hypentropy projection of $\by^{t+1}$ onto the
contraint set. In this adaptive update, $\nabla \phi_{\beta_{t-1}}$ is used to
map into the dual space where a gradient update occurs. Afterwards, $(\nabla
\phi)^{-1}_{\beta_{t}}$ maps back to the primal. When used in an OCO setting
over the norm-$1$ ball, $\beta_t$ can always be chosen to be sufficiently
small such that projection step is voided. $\EGpm$ fits into this algorithmic
paradigm with a specific choice of $\beta_t$ that avoids hypentropy
projection. In the setting of OCO over the norm-$1$ ball of radius $\beta d$,
we have the following result.
\begin{theorem}
    \label{HU-EGPM:theorem}
    $\EGpm$ with learning rate $\eta$ is equivalent to the adaptive $\HU$
    algorithm described in \eqref{ahu:eq} with the same learning rate and
    $\beta_t = {\beta d}\left({\sum_{i=1}^d\cosh(\eta \barg^t_i)}\right)^{-1}$,
    where $\barg^t = \sum_{s=1}^t\bg^t$.
\end{theorem}
\begin{proof}
    We start with some analysis of $\EGpm$. We have
    $u^{t+1}_i \propto \exp(-\eta \barg^t_i)$ and similarly
    $v^{t+1}_i \propto \exp(-\eta \barg^t_i)$. Normalizing the two such that
    $\|(\bu^{t+1}, \bv^{t+1})\|_1 = \beta d$ yields the normalization factor,
    \begin{align*}
        \frac{\beta d}{\sum_{i=1}^d\exp(-\eta \barg^t_i) +\exp(\eta \barg^t_i)} =
        \frac{2\beta d}{\sum_{i=1}^d\cosh(\eta \barg^t_i)}
    \end{align*}
    Putting the above all together, we get
    \begin{align}
        w^{t+1}_i = u^{t+1} - v^{t+1}
        &= \frac{2\beta d}{\sum_{i=1}^d\cosh(\eta \barg^t_i)}
        \big(\exp(-\eta \barg^t_i) - \exp(\eta \barg^t_i) \big) \nonumber \\
        &= \frac{\beta d \sinh(-\eta \barg^t_i)}{\sum_{i=1}^d\cosh(\eta \barg^t_i)} \label{expanded_egpm:eq}\\
        &= \beta_t \sinh(-\eta \barg^t_i) ~. \nonumber
    \end{align}
    Therefore, we have $\bw^{t+1} = \nabla \phi^{-1}_{\beta_t}(-\eta \barg^t)$.
    We can now show that the adaptive $\HU$ algorithm described in \eqref{ahu:eq}
    results in the same weights. We prove this property by induction on $t$. The
    base case follows because we initialize $\bw^{0} = 0$ in $\HU$. Now we assume
    that $\bw^{t} =  \phi^{-1}_{\beta_{t-1}}(-\eta \barg^{t-1})$. Applying the
    hypentropy update, we have
    \begin{align*}
        \by^{t+1} = \nabla \phi_{\beta_t}^{-1}(\nabla
        \phi_{\beta_{t-1}}(\phi^{-1}_{\beta_{t-1}}(-\eta \barg^{t-1})) - \eta
        \bg^{t}) = \nabla \phi_{\beta_t}^{-1}(-\eta \barg^{t-1} - \eta \bg^{t})
        = \nabla \phi^{-1}_{\beta_t}(-\eta \barg^t)~.
    \end{align*}
    Now note that since $\forall x \in \RR, |\sinh(x)| \leq \cosh(x)$, projection
    never takes place, so $\bw^{t+1} = \by^{t+1}$.
\end{proof}

We also find it useful to consider these updates without normalization or
projection. Without any constraint, the regularization parameter $\beta$ does
need to change.
\begin{theorem}
    \label{HU-EGPM-unnormalized:theorem}
    Running $\HU$ with a learning rate $\eta$ and a regularization parameter $\beta$
    without projection is equivalent to running $\EGpm$ without normalization.
\end{theorem}
\begin{proof}
    Note that in $\EGpm$ without normalization we get,
    \begin{align*}
        u^{t+1}_iv^{t+1}_i =
        u^{t}_i \exp(-\eta g^{t}_i)v^{t}_i \exp(\eta g^{t}_i) =
        u^{t}_iv^{t}_i ~.
    \end{align*}
    Therefore, $u^{t}_iv^{t}_i$ remains fixed and due to the initialization
    $\forall i , u^{0}_iv^{0}_i = \frac{\beta^2}{4}$. This inverse relationship
    between $\bu$ and $\bv$ can be used to find a simple closed form solution for
    $\bw$ by solving a quadratic equation. In particular, we know $u > 0$, so we have
    \begin{align*}
        w = u - v
        = u - \dfrac{\beta^2}{4u} ~\Rightarrow ~
        u = \dfrac{w + \sqrt{w^2 + \beta^2}}{2} ~.
    \end{align*}
    The resulting final update is
    \begin{align*}
        w^{t+1}_i &= \dfrac{\sqrt{(w^{t}_i)^2 +\beta^2} + w^{t}_i}{2}
        \exp(- \eta g^{t}_i) - \dfrac{\sqrt{(w^{t}_i)^2 +\beta^2} - w^{t}_i}{2}
        \exp(\eta g^{t}_i)\\
        &= \sqrt{(w^{t}_i)^2 +\beta^2} \dfrac{\exp(- \eta g^{t}_i) - \exp(\eta g^{t}_i)}{2}
        + w^{t}_i\dfrac{\exp(- \eta g^{t}_i) + \exp(\eta g^{t}_i)}{2} \\
        &= \sinh(- \eta g^{t}_i)\sqrt{(w^{t}_i)^2 +\beta^2} + \cosh(- \eta g^{t}_i)w^{t}_i ~.
    \end{align*}
    Now we consider HU algorithm with the same parameters,
    \begin{align*}
        \bw^{t+1} &= \nabla \phi_{\beta}^{-1}(\nabla \phi_{\beta}(\bw^{t}) - \eta \bg^{t})
        & \big[\mbox{Algorithm~\ref{hu:algorithm} update}\big]\\
        ~ ~ ~\Rightarrow w^{t+1}_i &= \beta\sinh\Big(\arcsinh\Big(\frac{w^{t}_i}{\beta}\Big)- \eta g^{t}_i \Big)\\
        &= \beta \Big[\sinh\Big(\arcsinh\Big(\frac{w^{t}_i}{\beta}\Big)\Big) \cosh(- \eta g^{t}_i)
        + \cosh\Big(\arcsinh\Big(\frac{w^{t}_i}{\beta}\Big)\Big)\sinh(- \eta g^{t}_i)\Big]\\
        &= \sinh(- \eta g^{t}_i)\sqrt{(w^{t}_i)^2 +\beta^2} + \cosh(- \eta g^{t}_i)w^{t}_i ~.
    \end{align*}
    Thus indeed the two updates with the conditions stated in the theorem are
    equivalent.
\end{proof}

\paragraph{Discussion}

While, we can represent $\EGpm$ as an adaptive variant of $\HU$, we still
would like to understand how $\HU$ with a fixed $\beta$ relates to $\EGpm$. A
brief look into $\beta_t$ provides some intuition that the two still should be
similar updates. Note that for small $\eta$,
$\cosh(\eta\barg^t_i)=1+O(\eta^2)$ and $\beta_t \approx \beta$.
In this regime, a fixed $\beta$ should result in a similar update.
The relation to $\EGpm$ also motivates the choice of $\beta\approx\tfrac{\|\bw^{\star}\|}{d}$
as this provides the right scale for $\EGpm$.
Another takeaway from Theorem\ref{HU-EGPM:theorem} is that the $\EGpm$ algorithm
viewed without doubling has an update that looks very much like RFTL.
For simplicity, let $\beta =
\tfrac{1}{d}$. We see from \eqref{expanded_egpm:eq} that the weights follow $
\bw^t=\nabla\phi^{*}(-\eta \barg^t)$ where
$\phi^{*}(\bx) = \log(\sum_{i=1}^d \cosh(x_i))$.